\documentclass[letterpaper]{article}
\usepackage{uai2020}
\usepackage[margin=1in]{geometry}

\usepackage{times}

\usepackage{microtype}
\usepackage{graphicx}
\usepackage{subfigure}
\usepackage{booktabs} 

\usepackage{enumitem}

\usepackage{dsfont}

\usepackage{amsfonts}       
\usepackage{algorithm}
\usepackage{algorithmicx}
\usepackage{algpseudocode}

\usepackage{graphicx,subfigure}
\usepackage{caption}
\usepackage{xspace}
\usepackage{xcolor}
\usepackage{comment}

\usepackage{amsthm,amssymb,amsmath,mathtools}
\newtheorem{theorem}{Theorem}

\newtheorem{lemma}[theorem]{Lemma}
\newtheorem{corollary}[theorem]{Corollary}

\newtheorem{definition}{Definition}[section]

\title{Coresets for Estimating Means and Mean Square Error with Limited Greedy Samples}

\author{ {\bf Saeed Vahidian} \\
Electrical and Computer Engineering \\
University of California San Diego\\
San Diego, CA 92093 \\
\And
{\bf Baharan Mirzasoleiman}  \\
Computer Science \\
Stanford University \\
Stanford, CA 94305\\
\And
{\bf Alexander Cloninger}   \\
Mathematics, HDSI \\
University of California San Diego    \\
San Diego, CA 92093
}

\begin{document}

\maketitle
\begin{abstract}
  In a number of situations, collecting a function value for every data point may be prohibitively expensive, and random sampling ignores any structure in the underlying data. We introduce a scalable optimization algorithm with no correction steps (in contrast to Frank–Wolfe and its variants), a variant of gradient ascent for coreset selection in graphs, that greedily selects a weighted subset of vertices that are deemed most important to sample. Our algorithm estimates the mean of the function by taking a weighted sum only at these vertices, and we provably bound the estimation error in terms of the location and weights of the selected vertices in the graph. In addition, we consider the case where nodes have different selection costs and provide bounds on the quality of the low-cost selected coresets. We demonstrate the benefits of our algorithm on the semi-supervised node classification of graph convolutional neural network, point clouds and structured graphs, as well as sensor placement where the cost of placing sensors depends on the location of the placement. We also elucidate that the empirical convergence of our proposed method is faster than random selection and various clustering methods while still respecting sensor placement cost. The paper concludes with validation of the developed algorithm on both synthetic and real datasets, demonstrating that it outperforms the current state of the art.
\end{abstract}

\section{INTRODUCTION}
In many problems in sociology, finance, computer science, and operations research, we have networks of interconnected entities and pairwise relations between them. 
A problem that arises often in practice is calculating the expected value of a variable in the form of the sum of the values of a smooth function on the nodes of a graph.
For example, in semi-supervised learning with Graph Neural Networks (GCNs), the generalization error is specified as the average of the loss functions associated with all the nodes in a graph.
Before the elections, opinion polls are usually designed to represent the opinions of a networked population 
about the candidates in expectation \cite{lippmann2017public}. 
In social networks, having a small average distance to other individuals in the network is considered a key factor to be influential \cite{social}.
In many environmental monitoring, knowing average temperature, humidity and water quality of various regions allow for taking preventive actions against forest fires and water contamination \cite{krause2008efficient, yu2005real}.
Finally, in health care, monitoring various health measures such as a population's average blood pressure, weight, cholesterol level, allow for designing health strategies and disease prevention actions \cite{Steinerberger2018}. 

In real-world networks containing millions of nodes and billions of edges, it is impractical to evaluate the function at every single node. 
Therefore, an important question is how to select a small representative 
subset (coreset) of nodes from a million-node graph such that the weighted sum of function values sampled at the nodes of the subset can be a good estimate of the sum of function values over the entire graph~\cite{Ortega2016}.
Another constraint that often arises in practice is that evaluating the function may incur a cost that is not necessarily equal for all the nodes in the graph. For example, placing sensors in certain areas may be more expensive than others \cite{sensor00}. Similarly, measuring blood pressure of people in hard-to-reach areas is more expensive. Hence, we wish to find a small representative weighted subset of nodes subject to a limited budget. 

There are main challenges in finding such a small coreset.
First, the selected subset and the corresponding weights should provide a bound on the estimation error of the first moment (mean) of the function at all the nodes in the graph. 
Moreover, the method should be simple to implement, computationally inexpensive, and have theoretical guarantees relating coreset size to both computational complexity and the quality of approximation.
Finally, different nodes may have non-uniform cost, and hence we need to be able to bound the error of estimating the mean while finding a low-cost solution.

%
Very recently, the authors in \cite{Steinerberger2018} considered this problem and provided an upper-bound on the estimation error of the mean of the function evaluated over the entire graph with a weighted subset of functions at nodes of an arbitrary subset.
The provided quadrature-type bound can be used to bound any mean estimation problem in which the function is sufficiently low-frequency, i.e., the function can be expressed in terms of a small number of eigenfunctions (with large eigenvalue) of a lazy random walk transition matrix on the underlying graph. This includes problems such as measuring average blood pressure in a database \cite{Steinerberger2018} and subsampled kernel two sample testing \cite{Cloninger2018}.
Intuitively, the placements that minimize the quadrature-type bound have the property that the random walks starting from every node in the subset and weighted by its corresponding weight, overlap very little.
However, the problem of finding the near-optimal subset and the associated weights by minimizing the upper-bound has remained unaddressed.

In this work, we address the question of finding a coreset of nodes that minimizes the estimation error of the expected function value over the entire graph, by minimizing the upper-bound provided in \cite{Steinerberger2018}.
Inspired by the recent work of \cite{Campbell18_ICML} on Bayesian coreset constructions, we propose a greedy algorithm to find a small subset of nodes and their weights that closely approximate the first moment of the function over the entire graph.
Moreover, we consider an extended problem of having each sampled node come with a non-uniform cost, a problem that arises in applications such as sensor placement, marketing, and other knapsack type problems.  We extend our algorithm to this setting, and characterize through a simple parameter the error in estimating the mean of the function one is willing to tolerate in order to seek a low cost solution.  

The paper is organized as follows.
In Section \ref{framework}, we describe the mathematical framework of our problem. 
In Section \ref{optimization}, we frame the greedy optimization algorithm for selecting points and weights, both for equal cost of placement and when there is a placement cost associated.  In Section \ref{theory}, we prove bounds on the convergence of our algorithm and bound the mean error of a smooth function in terms of the algorithmically selected points.  In Section \ref{examples}, we demonstrate the success of our algorithm over random sampling and several benchmark unsupervised learning, semi-supervised learning, and clustering algorithms for a number of different applications.

\section{RELATED WORK}
The ever increasing size of modern datasets motivated data reduction techniques as a preprocessing step to speed up subsequent optimization problems.
Existing graph summarization methods mainly focus on obtaining sparse subgraphs that can be used to approximate properties of the original graph (degree distribution, size distribution of connected components, diameter, or community structure).
Core techniques include graph clustering or community detection methods \cite{fortunato2010community, lancichinetti2009community}, bit compression-based methods \cite{navlakha2008graph}, sparsification-based \cite{spielman2011graph} and sketching methods \cite{ahn2012graph, liberty2013simple}, and influence-based methods \cite{mathioudakis2011sparsification}.
While these methods maintain structural properties of the original graph, they cannot guarantee that an algorithm working on the summary provides a solution close to the solution found based on the entire data. For instance, graph summarization algorithms cannot guarantee that the function sampled at the selected points has similar statistics to the function evaluated on the entire network.

%
In contrast, coresets are weighted subsets of the data, which guarantee that for the specific problem at hand, models fitting the coreset also provide a good fit for the original dataset. 
This approach has been successfully applied to a variety of problems including $K$-means and $K$-median clustering \cite{har2004coresets}, mixture models \cite{lucic2017training}, 
low rank approximation \cite{cohen2017input}, spectral approximation \cite{agarwal2004approximating, li2013iterative}, Nystrom methods \cite{agarwal2004approximating, musco2017recursive},
and Bayesian inference \cite{Campbell18_ICML}.
Coreset construction methods traditionally perform importance sampling with respect to sensitivity score, defined as the importance of the point with respect to the objective function we wish to minimize, to provide high-probability solutions \cite{har2004coresets, lucic2017training, cohen2017input}. Greedy algorithms, which are special cases of the Frank-Wolfe algorithm, were described more recently to provide worst-case guarantees for problems such as support vector machines \cite{clarkson2010coresets}, and Bayesian inference \cite{Campbell18_ICML}.
In this work, we propose a greedy algorithm to construct coresets for estimating the first moment of a function defined on the nodes of a large graphs.

\section{GENERAL FRAMEWORK}\label{framework}
Here, we aim to choose a subset of vertices and weights to be the best representative of the graph. More specifically, we assume that the dataset $V$ have some geometric structure encoded in a graph $G = (V,E)$, where $V$ and $E$ denote the set of nodes and edges. The problem is how to choose a subset $S \subseteq V$ of vertices and weights $w_s > 0$ for every $s \in S$ in order to approximate the mean $\mu_f$ with a weighted approximation $\widehat{\mu}_f$, where
\begin{align}\label{means}
\mu_f = \frac{1}{{\left| V \right|}}\sum\limits_{v \in V} {f\left( v \right)}, && \widehat{\mu}_f = \sum\limits_{s \in S} {{w_s}f\left( s \right)}.
\end{align}
The graph can either be given a priori, or constructed on a point cloud $V\subset \mathbb{R}^d$ via a kernel $K:V\times V\rightarrow \mathbb{R}_+$.  We must assume that the function $f:V\rightarrow \mathbb{R}$ must have some relationship to the graph, or else the optimal sampling would be a random search.  In our context, this assumption takes the form that the function can be expressed in terms of a small number of eigenfunctions (with large eigenvalue) of a lazy walk graph transition matrix on the graph G.  
\begin{definition}
The lazy random walk graph~\footnote{A lazy random walk is a walk on a graph with self loops. Here the probability of taking self loop is
inversely proportional to the degree of the node.} transition matrix $P$ on $G = (V,E)$ is constructed by
\begin{align*}
P = \frac{1}{d_{max}} (A - D) + I, 
\end{align*}
where $A$ denote the (weighted) symmetric adjacency matrix of $G$ such that $A_{ij}=A_{ji}\geq0$ is the weight of the edge connecting nodes $i$ and $j$, $D$ is a diagonal matrix with $D_{ii} = \sum_j A_{ij}$, $d_{max} = \max_i D_{ii}$, and $I$ is the identity matrix.
\end{definition}
\begin{definition}\label{def:lowfreq}
A function $f:V\rightarrow \mathbb{R}$ is in $P_\lambda$ for a lazy walk transition matrix $P$, with eigendecomposition $P = U\Lambda U^*$, if 
\begin{align*}
f = \sum_{|\lambda_i|>\lambda} b_i U_i,
\end{align*}
where $\lambda_i$ and $U_i$ are the eigenvalues and eigenvectors of $P$, $b_i>0$ is a weight, and $0 \leq \lambda \leq 1$ is a parameter controlling the degree of smoothness.
$P_\lambda$ is the subspace spanned by the eigenfunctions of $P$ associated with eigenvalue $\ge \lambda$. 
\end{definition}


Spectral smoothness on graphs, as in Definition \ref{def:lowfreq}, is necessitated by the fact that there is no traditional notion of gradients on graphs. 
There’s a large body of work that shows that eigenfunctions of the kernel with large eigenvalue are of lower frequency than eigenvectors with small eigenvalue \cite{coifman2006diffusion, hammond2011wavelets}. Thus spectrally band-limited functions must themselves be smooth \cite{Steinerberger2018}. Our ability to approximate the mean of the function decays as the frequency of the function increases, since the function becomes more chaotic. This is a fundamental limitation, as functions unrelated to the underlying graph structure cannot be approximated with a coreset in any way better than random sampling.

We can also consider the additional constraint that choosing every nodes $v \in V$ may come with a cost $C:V\rightarrow\mathbb{R}_+$.  We seek an algorithm that will choose the subset of nodes $S \subseteq V$ in a way that is 
\setlist{nolistsep}
\begin{itemize}[noitemsep]
\item Greedy in order to quickly choose and add additional points,
\item Minimizes the error in estimation of the mean of $f$, and
\item Incorporates the cost $C(v)$ by choosing low-cost points to the subset.
\end{itemize}

The motivation for this framework and use of the lazy walk graph transition matrix comes out of the work in ~\cite{Steinerberger2018}, in which it was noted that the mean error in $f$ can be bounded in terms of the choice of points and weights, as in the following

\begin{align}
\label{scxx}
\begin{split}
\forall f \in {P_\lambda }, &
 \left| {\frac{1}{n}\sum\limits_{v \in V} {f\left( v \right) - \sum\limits_{s \in S} {{w_s}f\left( s \right)} } } \right| \le \\
& {\left\| f \right\|_{{P_\lambda }}}\mathop {\min }\limits_{\ell  \in N} \frac{1}{{{\lambda ^\ell }}}{\left( {\left\| {{{P}^\ell }\sum\limits_{s \in S} {{w_s}{\delta _s}} } \right\|_2^2 - \frac{1}{n}} \right)^{\frac{1}{2}}},
\end{split}
\end{align}
for $\ell \in \mathbb{N}$, $0<\lambda<1$, and $w_s$ summing to 1. Moreover, $n$ is the number of vertices of the graph, and $\delta_s$ is the binary vector taking value zero at every index except for $s$.
This result implies that minimizing the right hand side in terms of $w_s$ and $S$ will yield a stronger bound on the moment estimation of $f$.
$P^\ell \delta_s$ is the probability distribution of a random walker starting in $s$ after $\ell$ jumps.
Therefore, intuitively the subset $S$ that minimize the quadrature-type bound have the property that the random walks starting from every node $s$ in the subset and weighted by its corresponding weight $w_s$ overlap very little.  

\subsection{PROBLEM DEFINITION}
Similar to the concept of duality in convex optimization, in order to achieve best results for the upper bound in \eqref{scxx} we shall minimize it. Therefore, the optimization problem can be formulated in the following from
%

\begin{equation}
\begin{aligned}
& \underset{w}{\text{minimize}}
& & {\left\| f \right\|_{{P_\lambda }}}\mathop  \frac{1}{{{\lambda ^\ell }}}{\left( {\left\| {{{P}^\ell }\sum\limits_{s \in S} {{w_s}{\delta _s}} } \right\|_2^2 - \frac{1}{n}} \right)^{\frac{1}{2}}}  \\
& \text{subject to}
& & \left| S \right| \le K, \; S \subset V, \\
&&& \sum_{s \in S} w_s = 1, \quad w_s > 0. 
\label{opp}
\end{aligned}
\end{equation}
where $K$ is the maximum number of selected vertices. Several discrete and continuous methods 
are relevant when solving the preceding problem in~\eqref{opp}.  However, the main issue with such an optimization is the constraint of choosing $S\subset V$, which leads to a difficult combinatorial optimization problem. 
For example, one approach related to~\eqref{opp} is that of computing a cardinality constrained minimization by solving sparse principal component analysis (PCA) problem \cite{d'Aspremont}. 
However, solving the sparse PCA optimization problem entails semidefinite relaxation 
and a greedy algorithm to calculate a full set of good solutions,
which is very expensive with total complexity of $O(n^3)$ \cite{jordan0}.

A better way to view this problem is in terms of an $L_2$-minimization problem on $P$. Because $\|f\|_{P_\lambda}$ and $\lambda$ are properties of the function $f$ we're analyzing (see Def. \ref{def:lowfreq}) and independent of the choice of points and weights, we will drop these terms in discussion of the optimization scheme, when appropriate.  Similarly, we will drop the dependence on $\ell$ for notational simplicity and simply deal with an arbitrary lazy random walk matrix $P$.  This is not an issue as $P^\ell$ is also a lazy random walk transition matrix, with eigenvalues $\lambda^\ell$.

We can also rewrite our cost as a matrix multiplication $Pw = P\sum_s w_s \delta_s$, 
and define our target function to be the normalized ones vector $\frac{1}{n}\mathds{1}$.  We note that the cost function of \eqref{opp} can be reframed by the above notational changes, as well as 
bringing the $\frac{1}{n}$ inside the norm, to arrive at
\begin{align*}
    {\left\| {{{P} }\sum\limits_{s \in S} {{w_s}{\delta _s}} } \right\|_2^2 - \frac{1}{n}} 
    &= \left\| Pw \right\|_2^2 +\frac{1}{n} - 2\left\langle Pw, \frac{1}{n}\mathds{1} \right\rangle\\
    &=  \left\|Pw - \frac{1}{n}\mathds{1}\right\|_2^2,
\end{align*}
where the first equality comes from the fact that $\langle Pw,\mathds{1}\rangle = 1$, and the second equality comes from the fact that $\|\frac{1}{n}\mathds{1}\|_2^2 = \frac{1}{n}$ and completing the square.
To this end,~\eqref{opp} can be posed in an equivalent form as
%
\begin{equation}
\begin{aligned}
& \underset{w}{\text{minimize}}
& & { \left\|Pw - \frac{1}{n}\mathds{1}\right\|_2 } \\
& \text{subject to}
&& \sum\limits_i \mathds{1}[w_i > 0] \le K\\
&&& w_i \ge 0
\label{op4}
\end{aligned}
\end{equation}
Problem \eqref{op4} shifts the original mean bound in \eqref{opp} into a constrained least squares problem for finding the optimal weights $w$.  Due to the shifting of $\frac{1}{n}\mathds{1}$ inside the norm, we are no longer constrained to have $\sum_{i\in V} w_i=1$ (discussed in detail in Appendix \ref{thm2proof}).  With that being said, our algorithm still has a weight normalization $\beta^*$ that arises in \eqref{eq:beta} and will push $\|w\|_1$ close to 1.

The optimization problem \eqref{op4} we construct has been considered previously \cite{pilanci2012recovery, kyrillidis2013sparse}, however these have limitations in the context of our graph problem.
Briefly, \cite{pilanci2012recovery} considers cardinality regularized loss function minimization subject to simplex constraints, which yields a computational complexity $O(n^4)$ for our graph problem. Similarly, \cite{kyrillidis2013sparse} has theoretical guarantees only under the restricted isometry property \cite{candes2008restricted}. Moreover, it assumes the size of the subset $K$ is known, which may vary in an a posteriori fashion in our applications.  Finally, we extend the literature by providing guarantees on the convergence rate explicitly in terms of the number of selected elements and the rate at which the error decreases as we sample more coreset points.

We also note that the
Problem \eqref{op4} can be solved 
by relaxing
the nonconvex cardinality constraint $\sum_i \mathds{1}[w_i > 0] \le K$ 
to a simplex constraint $\sum_i \|P_i\| w_i = \sum_i \|P_i\| $ for the columns $P_i$ of the matrix $P$, and using the
Frank–Wolfe (FW) algorithm that iteratively chooses the point most aligned with the residual error.
However, there are some problems for which FW performs very poorly for any number of iterations because FW must scale the objective function in Problem \eqref{op4} suboptimally by $\sum_{i} w_i\| P_i \|$ rather than $\|  \frac{1}{n}\mathds{1}\|$, in order to maintain feasibility~\cite{Campbell18_ICML}. In this paper, we mainly focus on the above-mentioned constrained optimization problem. In the following section, we provide a new radial optimization algorithm for problem~\eqref{op4} and demonstrate that it yields theoretical guarantees at a significantly reduced computational cost. More importantly, in contrast to FW and its extensions~\cite{Jaggi2015}, the algorithm developed in this work has no correction steps and geometric error convergence.

\section{OPTIMIZATION ALGORITHM}\label{optimization}
The optimization problem in \eqref{op4}, despite involving a least squares cost function, is nonconvex in $w$ due to the cardinality constraint. Inspired from~\cite{Campbell18_ICML}, without any loss of generality, the weights, $w$ could be scaled by an arbitrary constant $\beta \ge 0$ without affecting feasibility. This motivates rewriting \eqref{op4} as
\begin{equation}
\begin{aligned}
& \underset{w, \beta}{\text{minimize}}
& & { \left\|\beta Pw -  \frac{1}{n}\mathds{1}\right\|_2 } \\
& \text{subject to}
&& \sum\limits_i \mathds{1}[w_i > 0] \le K\\
&&& w_i \ge 0, \beta \ge 0
\label{op44}
\end{aligned}
\end{equation}
Following~\cite{Campbell18_ICML} we now begin by solving the optimization problem in $\beta$. 
We define $\beta^*$ as the solution to the problem \eqref{op44} given $w$ which can be computed analytically as
\begin{equation}\label{eq:beta}
\begin{aligned}
{\beta ^ * } &= \frac{\|{ \frac{1}{n}\mathds{1}}\|}{\|{P w }\|}\max \left\{ {0,\left(\frac{Pw}{\|{P w }\|}\right)^T\left(\frac{ \frac{1}{n}\mathds{1}}{\| \frac{1}{n}\mathds{1}\|}\right)} \right\} \\
&= \frac{1}{n \|Pw\|} \max \left\{ 0,\left(\frac{Pw}{\|{P w }\|}\right)^T \mathds{1} \right\}.
\end{aligned}
\end{equation}
Substituting $\beta^*$ in the  objective above and expanding the square, the weighted subset of nodes (coreset) can be found by solving:
%
%
\begin{equation}
\footnotesize
\begin{aligned}
& \underset{w}{\text{minimize}}
& & \frac{1}{n}\left( 1 - \max \left\{ {0,\left(\frac{Pw}{\|{Pw }\|}\right)^T\mathds{1}} \right\} \right) \\
& \text{subject to}
&& \sum\limits_i \mathds{1}[w_i > 0] \le K\\
&&&  w_i \ge 0
\label{op443}
\end{aligned}
\end{equation}
This result shows that the minimum of the problem in \eqref{op443} occurs by alignment of the vectors $Pw$ and $\frac{1}{n}\mathds{1}$ independent of their norm. 
Finally, we define 
$P(w) = \sum_i w_i \frac{P_i}{\|P_i\|}$ , and ${P^*} = \frac{{\frac{1}{n}\mathds{1}}}{{\left\| \frac{1}{n}\mathds{1} \right\|}} = \frac{1}{\sqrt{n}}\mathds{1}$.
With that in mind, we can reformulate \eqref{op443} in an equivalent maximizing problem as in the following
\begin{equation}
\begin{aligned}
& \underset{w}{\text{maximize}}
& & P(w)^T{P^*} \\
& \text{subject to}
&& \sum\limits_i \mathds{1}[w_i > 0] \le K\\
&&&{\left\| P(w)  \right\| = 1}\\
&&&  w_i \ge 0
\label{op44s3}
\end{aligned}
\end{equation}
According to the constraints in \eqref{op44s3}, we are optimizing the objective over a unit hypersphere rather than the simplex. Before solving the problem in \eqref{op44s3}, we extend it to a more general case where nodes have different selections costs, and then we provide a new algorithm for solving it.

\subsection{OPTIMIZATION WITH SELECTION COST} 

In many applications, selecting some reference points representing the whole data involve some factors such as cost of selection associated to each data. 
For example, placing sensors in certain regions may be more expensive than others. Similarly, measuring blood pressure of  people in hard-to-reach areas is more  expensive.
In problems such as these, it is natural to seek a trade-off between the two goals of minimizing the error (selecting nodes $S \subseteq V$ that maximize alignment of $P(w)$ and $P^*$) and the cost of choosing nodes in $S$. To this end, we incorporate a new parameter $C$ into the problem controlling the cost associated with each node. In what follows, we will focus on the reparameterized maximization problem, which can be written as:
%
%
\begin{equation}
\begin{aligned}
& \underset{w}{\text{maximize}}
& & P(w)^T{P^*} - \lambda C(S) \\
& \text{subject to}
&& |S|\le K \textnormal{ for } S=\{i:w_i>0\}\\
&&&{\left\|P(w)\right\| = 1}\\
&&& w_i \ge 0
\end{aligned}
\label{opt:giga}
\end{equation}

Now we provide a greedy algorithm, sample cost greedy iterative geodesic ascent (SCGIGA), for the above optimization problem.
At every iteration, the algorithm finds the point indexed by $v^*$ for which the geodesic between $P(w)$ and $P(v^*)$ is most aligned with the geodesic between $P(w)$ and $P^*$.
We then find the set of all vertices for which the alignment is within $\kappa$-percent of the alignment of $v^*$, for a parameter $0< \kappa \leq 1$. Among such points, we add the vertex with minimum cost to the solution.
Once the point has been added, the algorithm reweights the coreset and iterates. 
SCGIGA detailed in Algorithm \ref{alg:GIGA} outlines how to solve the optimization problem in~\eqref{opt:giga}. It is noteworthy that SCGIGA is a general algorithm which is also valid for the case where there are equal costs associated with every data point. This corresponds to $\kappa=1$ in our settings. 

A benefit of the greedy approach is that increasing the number of coreset points to $K+1$ simply requires adding the next point for geodesic ascent and marginally changing the weights.
This is unlike \cite{kyrillidis2013sparse}, in which changing to the $K+1$ simplex requires recomputing the optimization scheme and in no way guarantees keeping the previous $K$ selected data points.  Similarly, this formulation allows us to easily incorporate nonuniform cost of sampling points.

We note that the most expensive computation in Algorithm \ref{alg:GIGA} is $\langle P_v, P(w) \rangle$ across $v$, but that $P(w)$ is the sum of at most $K$ vectors and each loop only adds one additional vector. Thus we can store previous inner products and only take different weighted combinations on each iteration, so each loop only requires $n$ inner products. Hence, Algorithm \ref{alg:GIGA} has average complexity $O(Knm)$, where $m$ is the average sparsity of a column.

We will establish the convergence guarantees and rate in Section \ref{theory}, and connect this convergence to bounding the estimate of the mean of $f$ as in  \eqref{scxx}.  But first, we wish to establish that incorporating cost of sampling into the optimization does not significantly impact the resulting minimum value, and thus results in a close to optimal greedy bound on the error in estimating the mean of $f$.  The gap can be characterized by the one parameter $\kappa$ to be chosen by the user, and choosing the minimal cost point among the set of vertices similar to $v^*$ only scales the resulting bound by a factor of $\kappa$.

\begin{theorem}\label{thm:relaxedopt}
Let $C^k_{max}$ be the sum of the $k$ largest elements of $C$.  If we choose $\lambda \le  \frac{1-\kappa}{C^k_{max} \min\|P_i\| \sqrt{n}}$, then the solution to \eqref{opt:giga}, $P(w^*)$, satisfies
\begin{align*}
    P(w^*)^T P^* \ge \kappa \max_w P(w)^T P^*.
\end{align*}
\end{theorem}
The proof can be found in the Appendix. 

Theorem \eqref{thm:relaxedopt} implies we will never incur too much loss to the original objective by incorporating cost of selecting the nodes that minimize the error.  Similarly, this implies that we can make every greedy choice and step in whatever fashion is deemed best for cost, as long as the choice is within $\kappa$ of the optimal step direction.

\begin{algorithm}[t]
	\begin{algorithmic}[1]   
        \State  Initialization$~{w_0} \leftarrow 0$
        \State $\forall v\in V,$ $P_v\leftarrow \frac{P_v}{\|P_v\|} $ 
        \For{$k \in \left\{ {0,...,K} \right\}$}
             \State ${a_k} \leftarrow \frac{\mathds{1} - \langle \mathds{1},P(w_k) \rangle P(w_k)}{\left\| \mathds{1}- \langle \mathds{1},P(w_k) \rangle P(w_k) \right\|}$
             \State $\forall v \in V,{a_{kv}} \leftarrow \frac{P_v - \langle P_v,P(w_k) \rangle P(w_k)}{{\left\| P_v - \langle P_v,P(w_k) \rangle P(w_k) \right\|}}$ ~~~
            \State ${v^*}  \leftarrow  \arg {\max _{v \in V}}{a_k}^T{a_{kv}}$ ~~~~~~
            \State \Comment{find vertex that maximizes alignment}
            \State $W \leftarrow \left\{ {v \in V|{a_k}^T{a_{kv}} \ge \kappa {a_k}^T{a_{k{v^ * }}}} \right\}$ 
            \State \Comment{  find vertices within $\kappa$-percent of max }
            \State ${v_k} \leftarrow \arg {\min _{v \in W}}C_v$ 
            \State \Comment{  find vertex in $W$ with min cost } 
            \State $\zeta_0 = \langle \frac{1}{\sqrt{n}} \mathds{1}, P_{v_k} \rangle$
            \State $\zeta_1 = \langle \frac{1}{\sqrt{n}} \mathds{1}, P(w_k) \rangle$
            \State $\zeta_2 = \langle P_{v_k}, P(w_k) \rangle$
            \State ${\delta_k} \leftarrow \frac{\zeta_0 - \zeta_1 \zeta_2}{(\zeta_0 - \zeta_1 \zeta_2) + (\zeta_1 - \zeta_0 \zeta_2)}$
            \State \Comment{choose the step size}
            \State ${w_{k + 1}} \leftarrow \frac{{\left( {1 - {\delta _k}} \right){w_k} + {\delta _k}1_{{v_k}}^{}}}{{\left\| {\left( {1 - {\delta _k}} \right)P({w_k}) + {\delta _k}P_{{v_k}}^{}} \right\|}}$
            \State \Comment{  update the weight }
            \State \textbf{end}     
    \EndFor
    \State $w = \beta w_k$ \Comment{Scale weights by $\beta$}
    \State $S = \{v \in V| w_v > 0\}$
    \State Sample $f$ at nodes in $S$
    \State $\widehat{\mu}_f = \sum_{v\in S} w_v f(v)$
     \State \textbf{return} $\widehat{\mu}_f$, $w$
\end{algorithmic}
\caption{Algorithm of SCGIGA}
 \label{alg:GIGA}
\end{algorithm}


\section{THEORETICAL ASPECTS}\label{theory}

Here we examine the guarantees that Algorithm \ref{alg:GIGA} yields for bounding the error in estimating the mean of $f\in P_\lambda$, where $P_\lambda$ is the subspace spanned by the eigenfunctions of $P$ associated with eigenvalues $\ge \lambda$. 
We will derive the general theorem for arbitrary $\kappa$, and as a special case we recover the results when all the nodes have equal selection cost ($\kappa=1$). 


\begin{theorem}\label{thm:guarantees}
Let $f\in P_\lambda$ with mean $\mu_f$ as in \eqref{means}, and assume there is a cost for selecting every node $v \in V$, i.e., $C(v):V\rightarrow \mathbb{R}_+$ and a slack parameter $\kappa$. If we choose the set of points $S$ and weights $w_s$  using Algorithm \ref{alg:GIGA} such that $|S|=K$, then
\begin{align*}
    \left| \mu_f - \sum\limits_{s \in S} {{w_s}f\left( s \right)}  \right| \le  \frac{\left\| f \right\|_{{P_\lambda }}}{\lambda^\ell}  \frac{ \mathop \eta\mathop v_K}{\sqrt{n}},
\end{align*}
where $v_K = O((1 - \kappa^2 \epsilon^2)^{K/2})$ for some $\epsilon$ and $\eta = \sqrt{1-\kappa^2\max_{i\in V}\left\langle \frac{P_i}{\|P_i\|},\frac{1}{\sqrt{n}}\mathds{1}\right\rangle^2}$.
\end{theorem}
The proof can be found in the Appendix. 

The main ideas of this theorem are three-fold.  The first series of lemmas that must be proved are extensions of the core lemmas in \cite{Campbell18_ICML}, which establish that the error is a contractive map after each update, which is later used in a fixed point theorem to show convergence and rate.  The extensions we make here are: 1) generalize their results to graphical geometries, rather than the log-likelihood construction proposed in \cite{Campbell18_ICML}, and 2) allow for a $\kappa$ relaxation of the greedy choice and prove how this relaxation affects the contractive mapping.  

The second main idea behind this theorem is to show that the $\kappa$-relaxation does not significantly affect the fixed point argument to show convergence of the cost-aware greedy algorithm as the number of chosen points grows.  This again borrows from \cite{Campbell18_ICML} in this more general setting, but the argument effectively comes down to propagating the additional error accrued by the cost-aware algorithm.

The final main idea behind this theorem is connecting the coreset choice to the first moment quadrature bound in \cite{Steinerberger2018}.  This comes from a shifting of the bound in \eqref{scxx} to allow for weights to not necessarily have to sum to $1$, and the recognition that this bound for general points and weights chosen is equivalent to the bound being minimized in \eqref{op4}.

We wish to note that the bound established in Theorem \ref{thm:guarantees} may not be sharp for the first few points greedily sampled, but does become sharp asymptotically due to fixed point convergence.  The bound on the first moment estimate of $f$ in terms of \eqref{scxx} is close to sharp if there is a spectral gap at $\lambda$ \cite{Steinerberger2018}.

\begin{figure*}[!t]
    \centering
    \begin{tabular}{cc}
          \includegraphics[width=.26\pdfpagewidth,height=.15\textwidth]{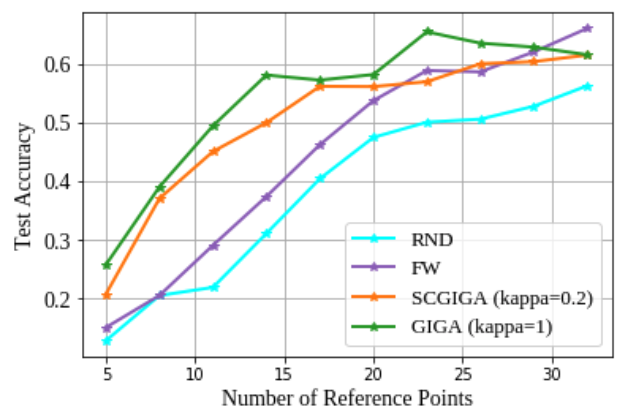} & \hspace{5mm}
          \includegraphics[width=.26\pdfpagewidth,height=.15\textwidth]{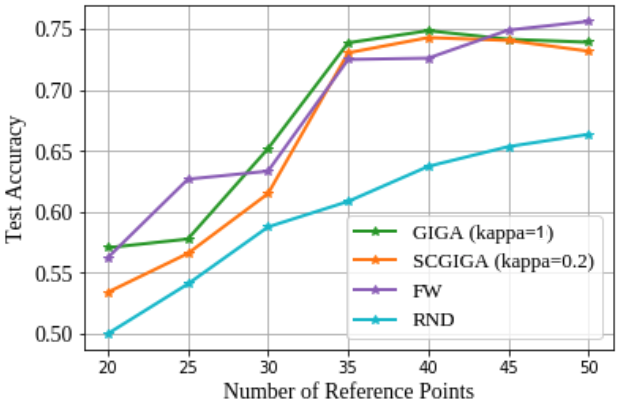} \\
          Stochastic block model & Cora citation dataset
    \end{tabular}
\vspace{-2mm}
    \caption{Classification accuracy obtained from a GCN model after a number of semi-supervised training iterations for different algorithms.}\label{fig:gcn}
\vspace{-1mm}
\end{figure*}

As a particular case of this theorem, when we always choose the optimal greedy node independent of cost, we recover the following guarantee.

\begin{corollary}
Let $f\in P_\lambda$, and choose the set of points $S$ and weights $w_s$  using Algorithm \ref{alg:GIGA} such that $|S|=K$.  Then
\begin{align*}
    \left| \mu_f - \sum\limits_{s \in S} {{w_s}f\left( s \right)} \right| \le \frac{\left\| f \right\|_{{P_\lambda }}}{\lambda^\ell}  \frac{\eta\mathop v_K}{\sqrt{n}},
\end{align*}
where $v_K = O((1 -  \epsilon^2)^{K/2})$ for some $\epsilon$ and $\eta = \sqrt{1-\max_{i\in V}\left\langle \frac{P_i}{\|P_i\|},\frac{1}{\sqrt{n}}\mathds{1}\right\rangle^2}$.
\end{corollary}
The proof of this corollary is a special application of Theorem \ref{thm:guarantees}, which is proved in the Appendix.

\section{EMPIRICAL EVIDENCE}\label{examples}

In this section, we evaluate our model on several sets of experiments. In order to compare both cases of our algorithm, i.e., when there are non-uniform costs on selecting different data points $(\kappa \ne 1)$ and when there are equal costs associated with every data point $(\kappa=1)$ we consider a fixed cost on each data randomly generated from a uniform distribution over $[0, 1)$. 
Similarly, in all examples involving point clouds, the generated graph comes from a $K$-nearest neighbor construction with $10$ nearest neighbors.

\subsection{GRAPH CNN CLASSIFICATION}\vspace{-2mm}

One example of the importance of bounding means comes in semi-supervised learning.  In such a problem, the goal is to approximate a function $h: V\to\mathbb{R}^m$ using a parametric model $h_\theta$ (such as the graph convolutional network).  The goal is to minimize 
$$L_{\theta,V}=\frac{1}{|V|}\sum_{v\in V} \|h_\theta(v)- h(v)\|^2,$$
where $L$ is a loss function, chosen judiciously depending upon the application.
However, in the event that sampling $h$ is expensive, the goal is to choose a coreset $S\subset V$ in order to estimate $L_{\theta,V}$ with a quadrature formula and the empirical risk
$$L_{\theta,S}=\sum_{v\in S}w_v \|h_\theta(v)- h(v)\|^2. $$
The quadrature error in this context, $|L_{\theta,S} - L_{\theta,V}|$, is exactly the \emph{generalization error}, and is now re-expressed as the estimation of the mean of the function $f(v) = \|h_\theta(v)- h(v)\|^2$.  Theorem \ref{thm:guarantees} applies in the situation that $h\in P_\lambda$ and that $h_\theta$ is sufficiently regularized to satisfy $h_\theta \in P_\lambda$.  We note that the $\|\cdot \|^2$ remains mostly low-frequency as the frequency of pointwise product of eigenfunctions can be bounded \cite{lu2019approximating}.
This implies that sampling $h$ only at $S$, and training $h_\theta$ on $S$, one can bound the generalization error on $V\setminus S$ using Theorem \ref{thm:guarantees}.  
 
Given the above discussion, we seek to test our algorithm on semi-supervised document classification in citation networks datasets, namely, Cora and semi-supervised node classification in a stochastic block model graph with 200 vertices and 10 clusters.  We train a two-layer graph convolutional neural network (GCN) as described in~\cite{gcn} and we follow the same pre-processing techniques as well. We compare against the same baseline methods as in the experiments presented in our paper tested on Cora dataset which is a real citation network dataset with $2,708$ nodes and $5,429$ edges as well as a stochastic cluster-based graph datasets.

A GCN takes a feature matrix and an adjacency matrix as inputs and for every vertex of the graph produces a vector, whose elements correspond to the score of belonging to different classes. 
An identity matrix is added to the original adjacency matrix in order to  enforce self loops.

\begin{figure*}
\centering
\begin{tabular}{ccc}
\includegraphics[width=.28\textwidth,height=.15\textwidth]{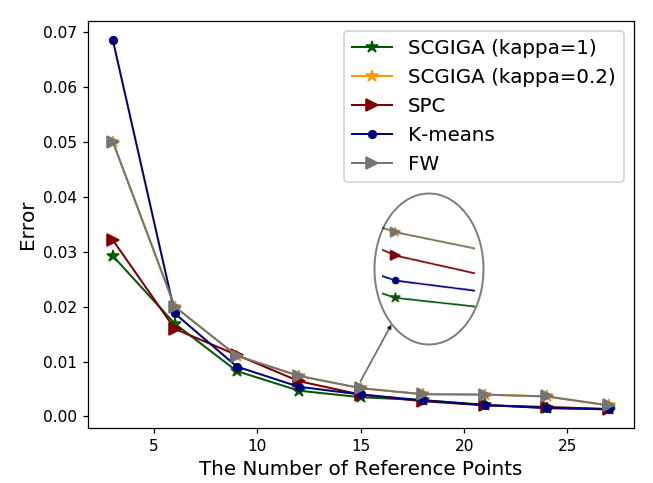} &
\includegraphics[width=.28\textwidth,height=.15\textwidth]{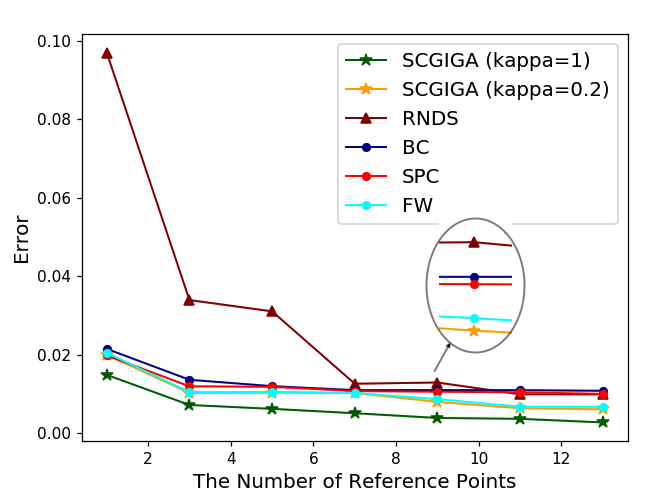} &
\includegraphics[width=.28\textwidth,height=.15\textwidth]{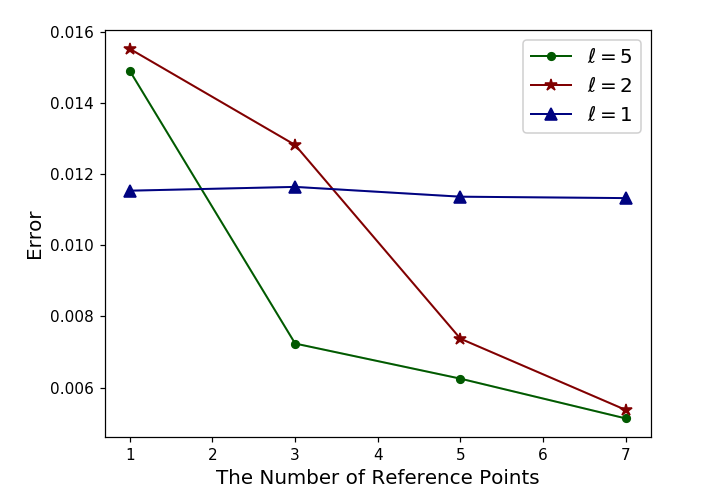} \\
Gaussian model & Stochastic Block Model & Stochastic Block Model for $\ell$
\end{tabular}
%
%
%
%
    \caption{The error comparison of estimating the mean of the function defined on the clusters (the two left ones), and the impact of the shaping parameter $\ell$ on the performance (Right one).}\label{fig:dd}
\end{figure*}

The neural network is trained in a semi-supervised setting, where the network is fed with the feature and adjacency matrices of the entire graph while the loss is only computed on the labeled vertices. Here the labeled vertices are the subset of vertices that are picked by the proposed method on the normalized adjacency matrix. We train both datasets for a maximum of 100 epochs using Adam~\cite{adam2015} with a learning rate of 0.01 and early stopping with a window size of 10, i.e. we stop training if the validation loss does not decrease for 10 consecutive epochs (as was done in \cite{gcn}). The results are summarized in Figure~\ref{fig:gcn}. 
Despite small fluctuations due to random initialization, as expected, the test accuracy tends to increase as more labeled points are utilized for training. Further, as can be seen from the figure our proposed algorithm, SCGIGA, has superior performance
in selecting the subset of data that comprises the most representative points of clusters. Lastly, because of the existence of outliers in a random graph, the accuracy of the proposed algorithm starts to improve slowly at about 60\% accuracy. However, we note that the model is trained with only 10\% of data which was considered to be the labeled ones, so this also implicitly suggests that our algorithm successfully picks out the most informative nodes.

\subsection{MEAN FUNCTION ON CLUSTERED DATA}

A standard unsupervised learning task is to learn clusters from data, either on a graph or a point cloud, and use those clusters to select points and weights for averaging a function. Standard clustering algorithms include $K$-means clustering and spectral clustering~\footnote{In multivariate statistics spectral clustering techniques gets rid of some of the eigenvalues of the similarity matrix of the data to perform nonlinear dimensionality reduction before clustering in fewer dimensions.}. 

The first experiment we run is on Gaussian model with three components. The components have the mean vectors of $[1~-3], [-3~~2], [3~~0]$, and all have the covariance matrix of the identity matrix, $I$. The components contain $20\%$, $30\%$, and $50\%$ of the data.  The function we choose to model is a simple smooth function that is an indicator function of the small cluster (1 on the small cluster, 0 on other clusters).
The results in Fig.~\ref{fig:dd} for the Gaussian Model show the error comparison of three algorithms including FW, $K$-means and spectral clustering (SPC) versus the number of clusters (centroids) or reference points. The performance of the algorithms is quantified by an error metric which is defined as the ${\rm{Err}} = {\left| {\sum_{\left| S \right| = k} {{a_s}f\left( S \right) - \mu_f} } \right|^2}.$   For fairness to comparable algortihms, $a_s$ in $K$-means, FW and SPC is defined as the ratio of the data in each cluster to the whole data, while $a_s=w_s$ in our algorithm.  Here, $f$ is the the indicator function on the small cluster and $\mu_f$ is the mean of the function. 

As is evident from the figure, our algorithm outperforms $K$-means, FW and spectral clustering for the whole range of the number of reference points. As can bee seen in this figure the maximum number of the reference data is $14$ out of $10000$ and our results reveal that the cost of the optimal solution ($C_{\rm{COS}}$) is $5.920$, while the cost we got with our cost aware algorithm, i.e., the cost of sub-optimal solution ($C_{\rm{CSO}}$) is $0.106$.
The more surprising observation in these figures is that even in the case where a fixed cost associated with each data in which results in our algorithm to yield a sub-optimal solution (the solution which is in the $\kappa$ percent of the optimal) our algorithm continues to do very well and work better than standard algorithms.  Also note that for only 3 reference points, which would lead to an ideal coreset of one point per cluster, SCGIGA considerably reduces the error compared to the competing algorithms.

We consider another experiment on clustered graphical data.
In Fig.~\ref{fig:dd} a stochastic block model with three clusters is designed where the first cluster contained $10\%$ of the data and the other two clusters contain $50\%$ and $40\%$ of the data. Then we define an indicator function on the small cluster, and we look for the average function value on these three clusters. As can be seen from these figures, our algorithm estimates the mean very well while random sampling (RNDS) needs more reference points to catch up the mean. More importantly, in this experiment by selecting $28$ reference data, $C_{\rm{COS}}=15.159$ while $C_{\rm{CSO}}= 0.075$.

Further, Fig.~\ref{fig:dd} sketches the impact of the shaping parameter, $\ell$ from \eqref{opp}, in our formulation on the error behavior of the same stochastic block model.  Clearly, incorporating a multi-step kernel yields much better performance in estimating the cluster coresets and weights.

\begin{figure*}[!h]
    \centering
    \begin{tabular}{cc}
          \includegraphics[width=.2\pdfpagewidth,height=.15\textwidth]{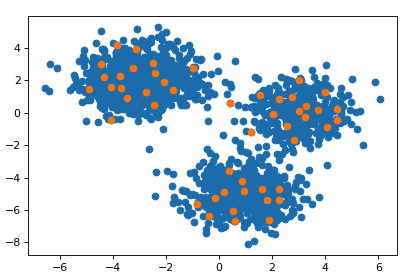} & 
          \includegraphics[width=.2\pdfpagewidth,height=.15\textwidth]{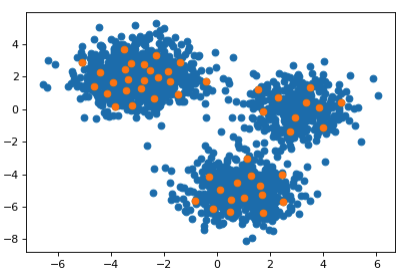} \\
          Cost associated & No cost associated
    \end{tabular}
%
    \caption{$2\%$ of coresets selected two cases, variable sampling cost with $\kappa=0.2$ (left) and uniform cost (right). The upper left-hand component, the middle component, and the lower component contain respectively $50\%$, $20\%$, $30\%$ of the data, and the average cost per data in each component is  $3.25$, $0.51$, and $2.04$,
    respectively. }\label{fig:animals}
\end{figure*}

\begin{figure*}[!h]
    \centering
    \begin{tabular}{ccc}
    \includegraphics[width=.28\textwidth,height=.15\textwidth]{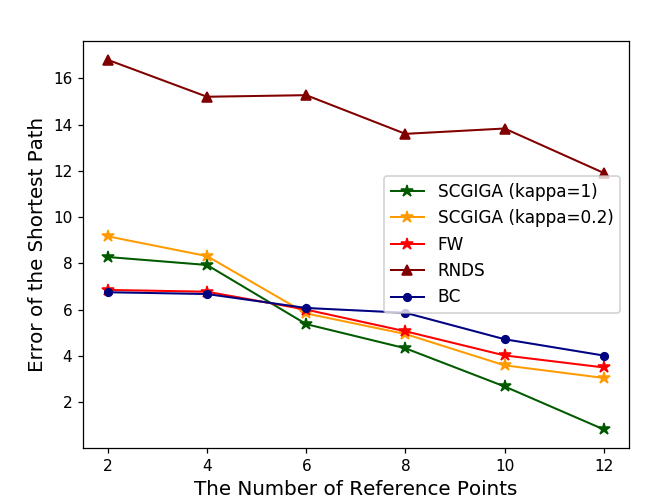} &
    \includegraphics[width=.28\textwidth,height=.15\textwidth]{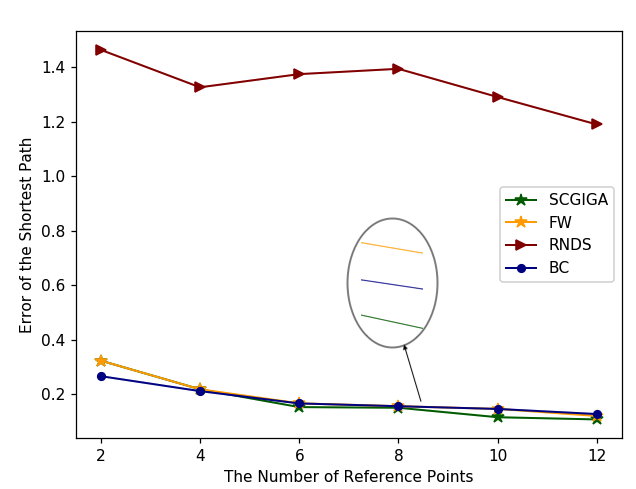} &
     \includegraphics[width=.28\textwidth,height=.15\textwidth]{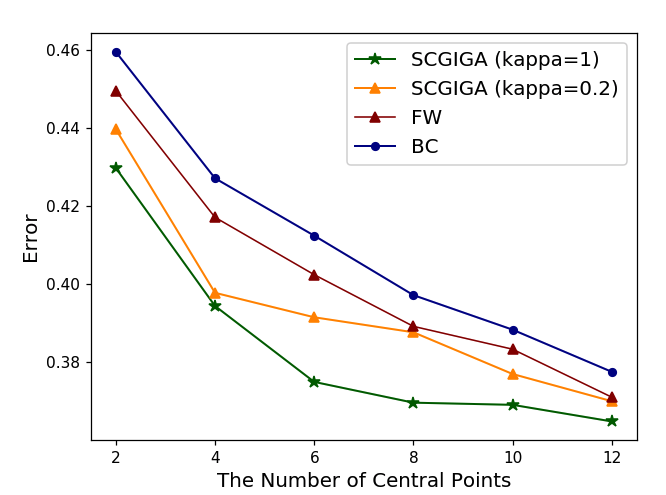}\\ 
A Powerlaw tree graph & A random generated graph & Facebook ego network 
    \end{tabular}
%
%
%
     \caption{Shortest path comparison of different algorithms on different graphs.}\label{fig:shortestpath}
\end{figure*}

Finally Fig.~\ref{fig:animals} shows three Gaussian clusters with the same mean vectors and covariance matrices as in the first experiment is considered. 
These two figures demonstrate the way that the proposed algorithm selects points when there is varying cost associated with each point and the special case of equal costs for choosing each point. In the figure, when there is a higher average cost of sampling on the largest cluster compared to the smallest,
the algorithm adapts to resample in a way not simply proportional to the number of points in each cluster.
In contrast, when no cost is associated the number of selected points are obviously less than that of the largest component.

\subsection{SHORTEST PATH ON GRAPH}

In this set of experiments, we assume that a graph $\mathcal{G}$ or an adjacency matrix is given. The goal is to find the average distance from a vertex to the rest of the graph, where the distance between two points is computed using Dijkstra’s algorithm~\footnote{Dijkstra’s algorithm is a greedy nature algorithm that looks for the minimum weighted vertex on every iteration.}. In this experiment, rather than computing the shortest path between a given node and every other node on the graph, we calculate the shortest path between the given node and just the reference points $S$.  This is of particular cost benefit on trees that have large diameter (maximum distance between two nodes) but low average path distance (e.g., trees, powerlaw graphs, etc).  

Fig.~\ref{fig:shortestpath} shows the results of the comparison of our algorithm, FW, Betweenness Centrality (BC) and random sampling (RS) on a Powerlaw Tree graph~\cite{Aiello2001}.  The error is defined as the difference between the weighted average distance of each vertex of the graph from the reference vertices and the average distance of each vertex from all the vertices of the graph. Fig.~\ref{fig:shortestpath} also demonstrates the same results for a randomly generated graph. In these two figures, both cases i.e., when fixed but different costs are associated with each data (vertex) and when equal costs on the data are included.

\textbf{Ego Networks~~}
In this experiment we consider a special type of network called an Ego Network. In an Ego Network, there is a “central” vertex (ego vertex) which the network highly depends on or revolves around. We consider the real-world Facebook Ego Networks dataset with 4,964 nodes~\cite{egonips}. The dataset contains the aggregated network of some users' Facebook friends. In this dataset, vertices represent individuals on Facebook, and edges between two users mean they are Facebook friends. The Ego Network connects a Facebook user to all of his Facebook friends and are then aggregated by identifying individuals who appear in multiple Ego Network. Algorithms such as betweenness centrality (BC) were proposed to measure the importance (centrality) of a user in the network. These algorithms select a user as a central one by looking at how many shortest paths pass through that user (vertex). The more shortest paths that pass through the user, the more central the user is in the Facebook network. We run our algorithm on the Facebook dataset to choose the central nodes. We then compare our algorithm with the BC and FW in terms of the error performance metrics described in the first experiment. The results in Fig.~\ref{fig:shortestpath} show that the developed algorithm in this work perform better in selecting the central points in the Facebook graph. 






\section{DISCUSSION AND CONCLUSIONS} \vspace{-2mm}
In this paper, we introduced a scalable, and theoretically-sound algorithm for 
minimizing the estimation error of the expected value of a function over entire graph.
Our proposed method can be applied to semi-supervised classification in graph neural networks (GCNs), social network graphs, point clouds, and on many other applications where a similarity metric between data points can be defined. 
We provided theoretical guarantees
on the convergence of the estimated mean and validated its efficiency empirically on real and synthetic datasets. Our work also opens up the notion of optimizing when there is a non-uniform cost of sampling, and provides a simple parameter for the trade off between accuracy and cost.

\section*{ACKNOWLEDGEMENTS}
AC was funded by NSF-DMS grants 1819222 and 2012266, and Sage Foundation Grant 2196.
BM was partially supported by SNSF P2EZP2\_172187.
\small{
\bibliographystyle{abbrv}

\bibliography{Ref_UAI}}

\onecolumn
\appendix   
\section{Proof of Theorem \ref{thm:relaxedopt}}

Because $P$ is a bistochastic matrix, and we know
$P^* = \frac{1}{\sqrt{n}}\mathds{1}$, we can lower bound 
\begin{align*}
    1\ge P(w)^T P^* &= \sum_i w_i \frac{P_i^T}{\|P_i\|} \frac{1}{\sqrt{n}}\mathds{1} \\
    &= \frac{1}{\sqrt{n}}\sum_i \frac{w_i}{\|P_i\|}\\
    &\ge \frac{1}{\min\|P_i\| \sqrt{n}}
\end{align*}
Similarly, $C(S)\le C^k_{max}$.  Now given $\lambda \le  \frac{1-\kappa}{C^k_{max} \min\|P_i\| \sqrt{n}}$, we compute
\begin{align*}
    \max_w P(w)^T P^* - \lambda C(S) &\ge \max_w P(w)^T P^* - \frac{1-\kappa}{\min\|P_i\| \sqrt{n}}\\
    &\ge \max_w P(w)^T P^*\left(1 - \frac{1-\kappa}{\min\|P_i\| \sqrt{n}} \frac{1}{ P(w)^T P^*} \right)\\
    &\ge max_w P(w)^T P^*\left(1 - \frac{1-\kappa}{\min\|P_i\| \sqrt{n}} \min\|P_i\|\sqrt{n} \right)\\
    &\ge \kappa \max_w  P(w)^T P^*.
\end{align*}

\section{Auxiliary Lemmas}
First, we make a few statements related to initialization of the process.   Lemma 3.4 from \cite{Campbell18_ICML} directly applies to this problem, and thus $\delta_k \in [0,1]$ $\forall k$.

\begin{lemma}
$\langle Pw_1,\frac{1}{n}\mathds{1} \rangle\ge \frac{\kappa}{\sqrt{n}\sum_i \|P_i\|}$
\end{lemma}
Proof follows equivalently to Lemma 3.1 from \cite{Campbell18_ICML}, with added caveat that our choice of weights is within $\kappa$ of maximum value.

\begin{lemma}\label{lemma:dcostbound}
The cost aware geodesic alignment $\langle a_t, a_{t,v_k}\rangle$ satisfies 
\begin{equation*}
  \langle a_k, a_{k,v_k}\rangle \ge \kappa \tau \sqrt{J_t}\vee f(t)  ,
\end{equation*}
for $$f(x) = \kappa\frac{\sqrt{1-x}\sqrt{1-\beta^2} \epsilon + \sqrt{x} \beta}{\sqrt{1 - \left(\sqrt{x}\sqrt{1-\beta^2}\epsilon -\sqrt{1-x}\beta\right)^2}}$$
and
$$\beta = 0 \wedge \min\langle \ell_n, \frac{1}{\sqrt{n}}\mathds{1}\rangle  \textnormal{ s.t. } \langle \ell_n, \frac{1}{\sqrt{n}}\mathds{1} \rangle>-1. $$
\end{lemma}
\begin{proof}
The lemma is equivalent to proving Lemma 3.6 in \cite{Campbell18_ICML} with one caveat.  Here our choice of node is $v_k$, which comes from choosing the cheapest cost node location from the set $S=\{v\in V| \langle a_k, a_{kv} \rangle \ge \kappa \langle a_k,a_{kv_k} \rangle\}$  Because of this, we can recover all results from $\langle a_k,a_{kv_k} \rangle$ with only a constant $\kappa$ in front, as our choice satisfies $\langle a_k,a_{kv_k} \rangle \ge \kappa \langle a_k,a_{kv_k} \rangle$.
\end{proof}

We apply Lemma \ref{lemma:dcostbound} to prove the following Theorem that is needed, and mirrors the results from \cite{Campbell18_ICML}.
\begin{theorem}\label{thm:giga}
Assume a cost of sensor placement $C(v):V\rightarrow \mathbb{R}_+$ and a slack parameter $\kappa$. If we choose the set of points $W$ and weights $a_w$  using Algorithm \ref{alg:GIGA} such that $|W|=K$, then
\begin{align*}
    \|Pw - \frac{1}{n}\mathds{1} \| \le  \frac{
     \eta\mathop v_K}{\sqrt{n}},
\end{align*}
where $v_K = O((1 - \kappa^2 \epsilon^2)^{K/2})$ for some $\epsilon$ and $\eta = \sqrt{1-\kappa^2 \max_{i\in V}\left\langle \frac{P_i}{\|P_i\|},\frac{1}{\sqrt{n}}\mathds{1}\right\rangle^2}$.
\end{theorem}
\begin{proof}
We mimic the results from \cite{Campbell18_ICML}, incorporating the additional cost parameter.  We
denote $J_k := 1 - \langle \frac{Pw_k}{\|Pw_k\|},\frac{1}{\sqrt{n}} \mathds{1} \rangle.$  If we substitute this into the formula for $\delta_t$, we get 
\begin{equation*}
    J_{k+1} = J_k(1 - \langle a_t, a_{kv_k} \rangle^2).
\end{equation*}
Applying our bound from Lemma \ref{lemma:dcostbound}, we get
\begin{equation*}
    J_{k+1} \le J_k\left(1 - \kappa^2\tau^2 J_k \right).
\end{equation*}
By applying the standard induction argument used in \cite{Campbell18_ICML}, we get 
\begin{equation*}
    J_k \le B(k) := \frac{J_1}{1 + \kappa^2\tau^2(k-1)}.
\end{equation*}
Because $B(k)$ still goes to 0, and $f(B(k))\rightarrow \kappa \epsilon$, there exists a $k^*$ such that $f(B(k))\ge \kappa \tau \sqrt{B(k)}$, and since $f$ is monotonic decreasing, $f(J_t) > f(B(k))$.  Using Lemma \ref{lemma:dcostbound}, we finish with
\begin{align*}
    J_k \le B(k\wedge k^*) \prod_{s=k^*+1}^k (1-f^2(B(s)))
\end{align*}
We note that $\frac{1}{n} J_k = \|\beta^* P(w) - P^*\|^2$, so this means 
\begin{align*}
    \|\beta^* Pw - \frac{1}{n}\mathds{1}\| \le \frac{\eta \mathop C_K }{\sqrt{n}}, 
\end{align*}
for constant $C_K$ combining the denominator in $B(k)$ and the product of $\prod_{s=k^*+1}{k} 1-f^2(B(s))$, and  $\sqrt{J_1} = \eta$.  Notice that $f(B(k))\rightarrow \kappa \epsilon$ shows a rate of decay of $v=\sqrt{1 - \kappa^2 \epsilon^2}$. 
\end{proof}

\section{Proof of Theorem \ref{thm:guarantees}}\label{thm2proof}
 We note that~\cite{Steinerberger2018} proves multiple bounds on $\left| {\frac{1}{n}\sum\limits_{v \in V} {f\left( v \right) - \sum\limits_{s \in S} {{w_s}f\left( s \right)} } } \right|$.  The main bound in the paper comes from using the fact that they assume $\sum_s w_s = 1$, which allows them to break up the inner product $\|P\sum_s w_s \delta_s - \frac{1}{n} {\mathds{1}}\|$ into its subsequent terms $\left(\|P\sum_s w_s \delta_w\|^2 - \frac{1}{n}\right)^{1/2}$.  We step away from this assumption and will instead work directly with the norm $\left\| {P\sum\limits_s {{w_s}} {\delta _s} - \frac{1}{n}{\rm{ \mathds{1}}}} \right\| = \left\| Pw - \frac{1}{n}\mathds{1} \right\|$. 
 
 By the same logic as in \cite{Steinerberger2018}, we know
 \begin{align*}
     \left| {\frac{1}{n}\sum\limits_{v \in V} {f\left( v \right) - \sum\limits_{s \in S} {{w_s}f\left( s \right)} } } \right| \le \frac{\|f\|_{P_\lambda}}{\lambda^\ell}\min_{\beta,w} \|\beta P^\ell w - \frac{1}{n}\mathds{1}\|.
 \end{align*}
We can simply replace $\widetilde{P} = P^\ell$ and inherit on $\widetilde P$ in Theorem \ref{thm:giga}, in particular that we still have $\sum_i \frac{1}{n}\widetilde P_i = \frac{1}{n}\mathds{1}$.  Thus, we can apply the guarantees of Algorithm \ref{alg:GIGA} and Theorem \ref{thm:giga} to bound $\|\beta P^\ell w - \frac{1}{n}\mathds{1}\| \le \frac{\eta \mathop v_K}{\sqrt{n}}$ and attain the desired result.

\end{document}